\documentclass[11pt]{amsart}

\usepackage{amsmath,amssymb,amsthm,amsrefs,a4wide}
\usepackage{latexsym}
\usepackage{indentfirst}
\usepackage{graphicx}
\usepackage{placeins}
\usepackage{booktabs}
\usepackage{algorithm}
\usepackage{algorithmic}
\usepackage{multirow}
\usepackage{color}

\theoremstyle{plain}
\newtheorem{theorem}{Theorem}

\theoremstyle{definition}

\theoremstyle{remark}
\newtheorem{remark}{Remark}

\DeclareMathOperator{\tr}{tr}

\DeclareMathOperator{\argmax}{argmax}

\newcommand{\bbm}{\begin{bmatrix}}
\newcommand{\ebm}{\end{bmatrix}}
\newcommand{\R}{\mathbb{R}}

\newcommand{\A}{\mathcal{A}}
\renewcommand{\t}{\top}

\newcommand{\E}{\mathbb{E}}

\newcommand{\grad}{\nabla}

\begin{document}

\title[A note on continuous-time online learning]{A note on continuous-time online learning}

\author[]{Lexing Ying} \address[Lexing Ying]{Department of Mathematics, Stanford University,
  Stanford, CA 94305} \email{lexing@stanford.edu}

\thanks{This work is partially supported by NSF grant DMS-2011699 and DMS-2208163.}



\begin{abstract}
In online learning, the data is provided in a sequential order, and the goal of the learner is to make online decisions to minimize overall regrets. This note is concerned with continuous-time models and algorithms for several online learning problems: online linear optimization, adversarial bandit, and adversarial linear bandit. For each problem, we extend the discrete-time algorithm to the continuous-time setting and provide a concise proof of the optimal regret bound.
\end{abstract}

\maketitle

\section{Introduction}\label{sec:intro}

In online learning, the data is provided in a sequential order, and the goal of the learner is to make online decisions to minimize overall regrets. This is particularly relevant for problems with a dynamic aspect. This topic has produced many surprisingly efficient algorithms that are nothing short of magic.

In most of the existing literature, online learning problems and algorithms are often placed in the discrete-time setting. There has been little work for online learning in the continuous-time setting. This note studies the continuous-time models for several important online learning problems
\begin{itemize}
\item online linear optimization,
\item adversarial bandit,
\item adversarial linear bandit.
\end{itemize}
For each problem, we propose a continuous-time model, describe an algorithm motivated by the discrete-time version, and provide a simple proof for the optimal regret bound. The main technical tools are Legendre transform and Ito's lemma.

{\bf Related work.} Several books, reviews, and lecture notes are devoted to online learning \cite{cesa2006prediction,lattimore2020bandit,bubeck2011introduction,hazan2022introduction,shalev2012online,rakhlin2009lecture,liang2016statistical} in the discrete-time setting. In recent years, there is a growing interest in the continuous-time setting \cite{portella2022continuous,fan2021diffusion,wager2021diffusion,kobzar2022pde,zhu2023continuous}. Our result for the adversarial bandit is closely related to the work in \cite{portella2022continuous}.

{\bf Contents.} The rest of this note is organized as follows. Section \ref{sec:legendre} summarizes the main results of the Legendre transform. Section \ref{sec:linearopt} discusses the online linear optimization problem. Section \ref{sec:adbandit} presents the continuous-time model for the adversarial bandit. Section \ref{sec:adlinearbandit} extends the result to the adversarial linear bandit.

\section{Legendre transform}\label{sec:legendre}

Let $X$ be a convex set in $\R^d$ and $F(x)$ be a convex function defined on $X$. To simplify the discussion, we assume that $F(x)$ is strictly convex.

The Legendre transform \cite{levi2014classical}, denoted by $G(y)$, of $F(x)$, is defined as
\[
G(y) \equiv \max_{x\in X} y^\t x - F(x),
\]
where the domain $Y$ of the set where $G(y)$ is bounded.

Let $x(y)$ be the point where the maximum is achieved for a given $y$. Then
\[
y = \grad F(x(y)).
\]
A key result of Legendre transform is that $F(x)$ is also the Legendre transform of $G(y)$:
\[
F(x) = \max_{y\in Y} x^\t y - G(y)
\]
and, similarly for a given $x$, the maximizer $y(x)$ satisifies
\[
x = \grad G(y(x)).
\] 
A trivial but useful inequality is
\[
F(x) + G(y) \ge x^\t y.
\]

In this note, we are concerned with the following case
\[
X = \Delta^d \equiv \left\{(x_1,\ldots,x_d): x_a\ge 0, \sum_{a=1}^d x_a = 1 \right\},\quad
Y = \R^d
\]
with $F(x)$ and $G(y)$ given by
\begin{equation}\label{eq:FG}
  F(x) = \beta^{-1} \sum_{a=1}^d x_a \ln x_a, \quad
  G(y) = \beta^{-1} \ln \left(\sum_{a=1}^d \exp(\beta y_a)\right),
\end{equation}
with $\beta>0$.

A direct calculation shows that
\[
(\grad G(y))_a = \frac{\exp(\beta y_a)}{\sum_c \exp(\beta y_c)} \equiv x_a
\]
and
\[
(\grad^2 G(y))_{ab} = \beta(\delta_{ab} x_a - x_a x_b).
\]

\section{Online linear optimization}\label{sec:linearopt}

{\bf Discrete-time problem.}
The discrete-time online linear optimization \cite{zinkevich2003online} is stated as follows.  At each round $t=1,\dots, T$
\begin{itemize}
\item The learner picks $x_t \in X=\Delta^d$.
\item The adversary picks reward $r_t \in [0,1]^d$.
\item The learner observes $r_t$ and gets reward $x_t^\t r_t$.
\end{itemize}

The regret is defined as
\[
R = \max_{x\in X} \sum_t (x-x_t)^\t r_t .
\]
\begin{remark}
  In most of the literature, the problem is formulated as minimizing the loss rather than maximizing the reward. These two formulations are equivalent. We chose the latter to put the problem into a Legendre transform setting.
\end{remark}

{\bf Discrete-time algorithm.}
An optimal algorithm for this problem is called follow-the-regularized-leader \cite{zinkevich2003online}. At each round $t=1,\ldots,T$, the learner chooses
\[
x_t = \argmax_{x\in X} \left( x^\t \left(\sum_{z=1}^{t-1} r_z\right) - F(x) \right)
\]
for $F(x)$ in \eqref{eq:FG}. By setting $\beta = \sqrt{\frac{2\ln d}{T}}$ in \eqref{eq:FG}, the discrete-time regret can be bounded by $O(\sqrt{T\ln d})$ \cite{lattimore2020bandit}.

{\bf Continuous-time problem.}
The continuous-time model is stated as follows. At each time $t \in[0, T]$,
\begin{itemize}
\item The learner picks $x(t) \in X$.
\item The adversary picks reward $r(t) \in [0,1]^d$.
\item The learner observes $r(t)$ and gets reward $x(t)^\t r(t)$.
\end{itemize}
By introducing the cumulative reward
\[
s(t) \equiv \int_0^t r(z) dz,
\]
the continuous-time regret is
\[
R = \max_{x\in X} \left( x^\t s(T) - \int_0^T x(t)^\t ds(t) \right).
\]

{\bf Continuous-time algorithm.}
Motivated by the discrete-time setting, we set the action at time $t$ as
\[
x(t) = \argmax_{x\in X}  \left(x^\t s(t) - F(x)\right).
\]
By Legendre transform,
\[
x(t) = \grad G(s(t)).
\]

{\bf Continuous-time regret bound.} The regret analysis is particularly simple in the continuous-time setting.

\begin{theorem}
  For any $\beta>0$, the continuous-time regret is bounded by $\beta^{-1} \ln d$.
\end{theorem}
\begin{proof}
  For any $x\in X$
  \begin{align*}
  x^\t s(T) - \int_0^T x^\t ds(t) &= x^\t s(T) - \int_0^T \grad G(s(t))^\t ds(t) \\ &= x^\t s(T) - G(s(T)) + G(0) \le F(x) + G(0) \le \beta^{-1} \ln d.
  \end{align*}
  Here we used the facts that $x^\t s \le G(s) + F(x)$, $F(x) \le 0$, and $G(0) = \beta^{-1} \ln d$.
\end{proof}

\begin{remark}
  By letting $\beta$ approach infinity, the regret goes to zero. This shows that in the continuous-time case, following the leader is, in fact, the optimal strategy. This is different from the discrete-time setting.
\end{remark}

\section{Adversarial bandit}\label{sec:adbandit}

{\bf Discrete-time problem.}  The discrete-time setting is stated as follows. The arms are indexed by $\{1,\ldots,d\}$.  At the beginning, the adversarial chooses the rewards $r_1,\ldots, r_T \in [0,1]^d$. In each round $t=1,\dots, T$
\begin{itemize}
\item The learner picks arm $a_t$.
\item The learner gets reward $r_{t,a_t}$ (but without knowing other components of $r_t$).
\end{itemize}
Since the arm can be chosen randomly, the regret is defined as
\[
R = \max_a \E \sum_t (r_{t,a} - r_{t,a_t}).
\]

{\bf Discrete-time algorithm.} The algorithm performs two tasks at each round $t$: (1) computing a probability distribution $p_t$ for selecting the arm and (2) forming an estimate $\hat{r}_t$ of $r_t$ based on the only reward $r_{t,a_t}$ received. Assuming that $\hat{r}_z$ are available for time $z<t$, the algorithm defines the probability
\[
p_{t,a} \propto \exp\left(\beta \sum_{z=1}^{t-1} \hat{r}_{z,a}\right)
\]
for an appropriate $\beta$ value and the reward estimate
\[
\hat{r}_{t,a} =
\begin{cases}
  \frac{r_{t,a_t}}{p_{t,a}}, & a=a_t\\
  0,  & \text{otherwise}.
\end{cases}
\]
From the properties of the Legendre transform,
\[
p_t = \grad G\left(\sum_{z=1}^{t-1} \hat{r}_z\right).
\]
The random estimate $\hat{r}_t$ is unbiased
\[
\E \hat{r}_t = r_t
\]
and its covariance matrix $\Sigma_t$ has entries given by
\[
\Sigma_{t,ab} = \frac{r_{t,a}^2}{p_a} \delta_{ab} - r_{t,a} r_{t,b}.
\]
By setting $\beta = \sqrt{\frac{2\ln d}{dT}}$, the discrete-time regret can be bounded by $\sqrt{2Td\ln d}$ \cite{lattimore2020bandit}.

{\bf Continuous-time problem.}  The continuous-time model is stated as follows. At the beginning, the adversarial chooses the rewards $r(t) \in [0,1]^d$ for $0\le t \le T$. At each time $t\in [0,T]$
\begin{itemize}
\item The learner picks arm $a(t)$.
\item The learner gets reward $r(t)_{a(t)}$ (but without knowing other components of $r(t)$).
\end{itemize}
The continuous-time regret is defined as
\[
R = \max_a \E \int_0^T (r(t)_a - r(t)_{a(t)}) dt.
\]

{\bf Continuous-time algorithm.}
Motivated by the discrete-time algorithm, we adopt the reward estimate
\[
\hat{r}(t)_a =
\begin{cases}
  \frac{r(t)_{a(t)}}{p(t)_a}, & a=a(t)\\
  0,  & \text{otherwise}.
\end{cases}
\]
Then the {\em cumulative reward estimate} $s(t)\in \R^d$ follows the following stochastic differential equation (SDE) \cite{oksendal2013stochastic},
\[
ds(t) = r(t) dt + \sigma(t) dB(t),
\]
where $\sigma(t) \sigma(t)^\t = \Sigma(t)$ with $\Sigma(t)_{ab} = \frac{r(t)_a^2}{p(t)_a} \delta_{ab} - r(t)_{a} r(t)_{b}$. The probability of choosing arm $a$ at time $t$ is 
\[
p(t)_a \propto \exp(\beta s(t)_a).
\]
Notice that $p(t) = \grad G(s(t))$.

{\bf Continuous-time regret bound.} The following theorem states the regret bound of the continuous-time algorithm after optimizing $\beta$.
\begin{theorem}
  For $\beta = \sqrt{\frac{2\ln d}{dT}}$, the continuous-time regret is bounded by $\sqrt{2Td\ln d}$.
\end{theorem}
\begin{proof}
  For an arbitrary arm $a$, let $x=(0,\ldots,1,\ldots,0)^\t$ with $1$ at the $a$-th slot. The regret with respect to $a$ can be recast as 
  \[
  \E \left[x^\t s(T) - \int_0^T \grad G(s(t))^\t ds(t) \right].
  \]
  In order to bound $\int_0^T \grad G(s(t))^\t ds(t)$, we use Ito's lemma
  \[
  dG(s) = \grad G(s)^\t ds + \frac{1}{2} ds^\t  \grad^2 G(s) ds.
  \]
  The second (quadratic variation) term can be written as  
  \[
  \frac{1}{2} \tr \left( ds ds^\t \grad^2 G\right) = 
  \frac{\beta}{2} \tr \left[ \left( \frac{r(t)_a^2}{p(t)_a} \delta_{ab} - r(t)_{a} r(t)_{b} \right)_{ab}
    \left(\delta_{ab} p(t)_a -p(t)_a p(t)_b\right)_{ab} \right] dt.
  \]
  Using the facts that both matrices are symmetric nonnegative definite and that the product only increases if one makes the matrices more positive definite, we can bound this by
  \[
  \frac{\beta}{2} \tr \left[ \left( \frac{r(t)_{a}^2}{p(t)_a} \delta_{ab}  \right)_{ab} \left( \delta_{ab} p(t)_a \right)_{ab} \right] dt
  \le \frac{\beta d}{2} dt,
  \]
  where  we used $r(t)\in [0,1]^d$. From this, we can bound the regret by
  \begin{align*}
    \E \left[ x^\t s(T) - \int_0^T d G(s(t)) \right]  + \frac{\beta dT}{2} &=  \E \left[x^\t  s(T) - G(s(T)) \right] + G(0) + \frac{\beta dT}{2} \\
    & \le F(x) + G(0) + \frac{\beta dT}{2} \le \beta^{-1} \ln d + \frac{\beta dT}{2}.
  \end{align*}
  By choosing $\beta = \sqrt{\frac{2\ln d}{dT}}$, the regret can be bounded by $\sqrt{2Td\ln d}$.
\end{proof}

\section{Adversarial linear bandit}\label{sec:adlinearbandit}

In practice, there might be many but correlated arms. A common setting is an arm set $\A=\{a\}\subset \R^d$ with $|\A|=k \gg d$. Assume that each arm $a \in \R^d$ satisfies $\|a\|_1 \le 1$.

{\bf Discrete-time problem.}  At the beginning, the adversarial chooses the rewards $r_1,\ldots, r_T \in [0,1]^d$. In each round $t=1,\ldots, T$
\begin{itemize}
\item The learner picks arm $a_t \in \A$.
\item The learner gets reward $a_t^\t  r_t$ (but without knowing other information about $r_t$).
\end{itemize}
The discrete-time regret is defined as
\[
R = \max_{a\in\A} \E \sum_t (a-a_t)^\t r_t.
\]

{\bf Discrete-time algorithm.} The discrete-time algorithm performs two tasks at round $t$: (1) computing a probability distribution $p_t$ for selecting an arm and (2) forming an estimate $\hat{r}_t$ for $r_t$ based only on $a_t^\t  r_t$.  Assuming that $\hat{r}_z$ are available for time $z<t$, the algorithm defines the probability
\[
p_{t,a} \propto \exp\left(\beta \sum_{z=1}^{t-1} \hat{r}_z^\t a \right)
\]
for some $\beta>0$ and the reward estimate
\[
\hat{r}_t = Q_t^{-1} a_t (a_t^\t  r_t) \quad\text{with}\quad Q_t = \sum_{a\in \A} p_{t,a} a a^\t .
\]
The random estimate $\hat{r}_t$ is unbiased
\[
\E \hat{r}_t = r_t
\]
and it has a covariance matrix $\Sigma_t \in \R^{d\times d}$ given by
\[
\Sigma_t = \sum_{a\in\A} p_{t,a} Q_t^{-1} a (a^\t r_t) (r_t^\t a) a^\t  Q_t^{-1} - r_t r_t^\t .
\]
By setting $\beta = O\left(\sqrt{\frac{\ln k}{dT}} \right)$ and including appropriate exploration, the regret can be bounded by $O(\sqrt{Td\ln k})$ \cite{lattimore2020bandit}.  Notice that it depends only logarithmically on the number of arms $k$.

{\bf Continuous-time problem.}  The continuous-time model is given as follows. At the beginning, the adversarial chooses the rewards $r(t) \in [0,1]^d$ for $0\le t \le T$. At each time $t \in [0,T]$
\begin{itemize}
\item The learner picks arm $a(t)$.
\item The learner gets reward $a(t)^\t r(t)$ (but without knowing other information of $r(t)$).
\end{itemize}
The continuous-time regret is defined as
\[
R = \max_{a\in A} \E \int_0^T (a-a(t))^\t r(t) dt.
\]

{\bf Continuous-time algorithm.}
Motivated by the discrete-time algorithm, we use the reward estimate
\[
r(t) = Q(t)^{-1} a(t) (a(t)^\t r(t)) \quad\text{with}\quad Q(t) = \sum_{a\in \A} p(t)_a a a^\t .
\]
Then the {\em cumulative reward estimate} $s(t)\in \R^d$ follows the following stochastic differential equation
\[
ds(t) = r(t) dt + \sigma(t) dB(t),
\]
where $\sigma(t) \sigma(t)^\t = \Sigma(t)$ with $\Sigma(t) = \sum_a p(t)_a Q(t)^{-1} a (a^\t r(t)) (r(t)^\t a) a^\t Q(t)^{-1} - r(t) r(t)^\t $.  The probability of choosing arm $a$ at time $t$ is
\[
p(t)_a \propto \exp(\beta a^\t s(t)).
\]

{\bf Continuous-time regret bound.}  Let $A$ be the $k\times d$ matrix with rows given by $a^\t $. Instead of defined over $\R^d$, we redefine $F(x)$ and $G(x)$ over $\R^k$
\[
F(x) = \beta^{-1} \sum_{i=1}^k x_i \ln x_i, \quad
G(y) = \beta^{-1} \ln \left(\sum_{i=1}^k \exp(\beta y_i)\right).
\]
Notice that $G(0) = \beta^{-1} \ln k$ and now $p(t) = \grad G(A s(t)) \in \R^k$.

\begin{theorem}
  For $\beta = \sqrt{\frac{2\ln k}{dT}}$, the continuous-time regret is bounded by $\sqrt{2Td\ln k}$.
\end{theorem}
\begin{proof}
  For an arbitrary arm $a$, let $x=(0,\ldots,1,\ldots,0)^\t$ with $1$ at the $a$-th slot. The regret with respect to $a$ can be written as
  \[
  \E \left[x^\t  A s(T) - \int_0^T \grad G(A s(t))^\t d (A s(t)) \right].
  \]
  To bound $\int_0^T \grad G(A s(t))^\t d(As(t))$, we again invoke Ito's lemma
  \[
  dG(As) = \grad G(s)^\t Ads + \frac{1}{2} ds^\t  A^\t  \grad^2 G(As) A ds.
  \]
  The second (quadratic variation) term can be written as (hiding the $t$ dependence)
  \begin{align*}
    \frac{1}{2} \tr \left( A ds ds^\t  A^\t \grad^2 G(As) \right) dt &=    \frac{\beta}{2} \tr \left( A^\t  (\delta_{ab} p_a -p_a p_b)_{ab} A  \left(\sum_a p_a Q^{-1} a (a^\t r) (r^\t a) a^\t  Q^{-1} \right)  \right) dt\\
    & \le  \frac{\beta}{2} \tr \left( A^\t  (\delta_{ab} p_a)_{ab}  A              \left( Q^{-1} \sum_a a p_a a^\t  Q^{-1}\right) \right)  dt \le \frac{\beta d}{2} dt.
  \end{align*}
  Here we used $\|a^\t r\|\le 1$, $A^\t  (\delta_{ab} p_a)_{ab}  A=Q$, and $\sum_a a p_a a^\t =Q$.

  From this, we can bound the regret by
  \begin{align*}
    \E \left[ x^\t  A s(T) - \int_0^T dG(As(t)) \right] + \frac{\beta dT}{2} &=  \E \left[ x^\t  A s(T) - G(A s(T)) \right] + G(0) + \frac{\beta dT}{2} \\
    &\le F(x) + G(0) + \beta dT/2 \le \beta^{-1} \ln k + \frac{\beta dT}{2}.
  \end{align*}
  By choosing $\beta = \sqrt{\frac{2\ln k}{dT}}$, the regret can be bounded by $\sqrt{2Td\ln k}$.
\end{proof}
Notice again that the bound depends only logarithmically on the number of arms $k$.

\section{Discussions}\label{sec:disc}

This note discusses continuous-time formulations and algorithms for several online learning problems. The main advantage of the continuous-time approach is that the proof of the regret bounds can be made concise. Many other online learning problems can be addressed similarly, including online convex optimization, semi-bandits, combinatorial bandits, and stochastic bandits. 

\bibliographystyle{abbrv}

\bibliography{ref}

\end{document}